\documentclass[letterpaper]{article}
\usepackage{aaai}
\usepackage{times}
\usepackage{helvet}
\usepackage{courier}
\usepackage{graphicx}
\usepackage{booktabs}
\usepackage{amsmath}
\usepackage{amssymb}
\usepackage{amsfonts}
\usepackage{amsthm}
\usepackage{bm}
\usepackage{multirow}
\usepackage{array}
\usepackage{xcolor}
\usepackage{url}
\usepackage{enumitem}
\usepackage[ruled,vlined]{algorithm2e}

\newtheorem{proposition}{Proposition}

\newtheorem{corollary}{Corollary}

\newtheorem{assumption}{Assumption}

\newcommand{\method}{\textup{\textsc{CoDA}}}
\newcommand{\R}{\mathbb{R}}
\newcommand{\E}{\mathbb{E}}

\newcommand{\Dcal}{\mathcal{D}}
\newcommand{\Lcal}{\mathcal{L}}
\newcommand{\Pcal}{\mathcal{P}}
\newcommand{\Scal}{\mathcal{S}}
\newcommand{\Jcal}{\mathcal{J}}
\newcommand{\aff}{\bm{A}}

\title{Co-Adaptive Multi-Task LoRA: Transfer-Aware, Label-Free\\ Control of Domain Participation}

\author{
    Wei Zhang\textsuperscript{\rm 1},
    Lin Tang\textsuperscript{\rm 1},
    Ming Zhao\textsuperscript{\rm 2},
    Yuxuan Wang\textsuperscript{\rm 2}\\
    \textsuperscript{\rm 1}Sichuan University, Chengdu, China\\
    \textsuperscript{\rm 2}University of Electronic Science and Technology of China, Chengdu, China
}

\nocopyright

\begin{document}
\maketitle

\begin{abstract}
Fine-tuning a single low-rank adapter on many domains at once is multi-task learning: the domains must be \emph{co-learned}, and how they share the adapter decides whether they help or hurt one another. Most efficient fine-tuning pipelines ignore this and train on a fixed, uniform mixture, leaving two coupled questions unanswered: how much should each domain participate, and which domains should be co-trained given that some transfer positively and others interfere? We show that both answers can be read off cheaply and without labels. A forward pass of the current shared adapter over a small unlabeled probe yields, per domain, a \emph{competence} signal whose level tracks remaining headroom and whose trajectory tracks learning speed; the drift of these probe representations yields a signed \emph{cross-domain affinity} that predicts pairwise transfer. We fold both into \method{}, a co-adaptive controller that solves a small entropy-regularized quadratic program on the simplex to set each domain's \emph{participation}---jointly its loss weight and its share of the sampled data---rewarding high-headroom, still-learning, mutually synergistic domains and damping interfering ones. The controller is forward-only, adds no trainable parameters, and wraps any multi-task LoRA pipeline. Across five heterogeneous domains and two backbones, \method{} improves the average over uniform mixing, learned mixtures, gradient-surgery multi-task optimizers, and online data selection while using half the data, and lowers cross-domain gradient conflict. We prove that the competence signal tracks domain risk, that the participation program has a unique fixed point reached by a contraction, and that its solution performs transfer-aware water-filling; analysis, ablations, and controls corroborate each claim.
\end{abstract}

\section{Introduction}
\label{sec:intro}

Adapting a pretrained large language model (LLM) to many downstream domains at once---knowledge, mathematics, code, reasoning, medicine---is now routine, and low-rank adaptation (LoRA)~\cite{hu2022lora} makes it cheap by training a single small adapter on top of frozen weights. When one shared adapter must serve all domains, fine-tuning is inherently \emph{multi-task learning} (MTL): the domains are co-learned through a common set of parameters, and the classic MTL phenomena appear---some domains reinforce each other while others compete, so training them together can help or hurt~\cite{crawshaw2020multi,standley2020tasks,fifty2021efficiently}. How the shared adapter is divided among the domains therefore determines the outcome.

Yet most efficient fine-tuning pipelines sidestep this. Data-selection methods score individual examples and keep a high-value subset~\cite{albalak2024survey,xia2024less,mindermann2022prioritized,wang2024greats}, but they operate on a single pooled stream and inherit a fixed, usually uniform or size-proportional, domain mixture. This leaves two \emph{coupled} questions unanswered: (i) \emph{how much should each domain participate}---how much of the shared capacity and the finite data-and-compute budget it should claim---and (ii) \emph{which domains should be co-trained}, given that co-training synergistic domains transfers positively while co-training conflicting ones injects interference~\cite{standley2020tasks,fifty2021efficiently,yu2020gradient}. Treating the mixture as fixed answers neither, and answering them well is exactly the MTL problem of balancing tasks and choosing task groupings, now posed inside a single LoRA under a budget.

Both questions have moving targets. First, domains have strongly \emph{heterogeneous and non-stationary} learning curves (Figure~\ref{fig:motivation}a): a domain the base model already handles well saturates after a few steps, while others keep improving, so the right amount of participation drifts during training. Second, transfer between domains is \emph{uneven and signed}: some pairs share structure and should be learned together, others interfere and should be kept apart (Figure~\ref{fig:motivation}c). A controller must read both the per-domain state and the between-domain relationships, and it must do so continually.

The apparent obstacle is that measuring a domain's remaining headroom or its transfer to others seems to require labeled held-out evaluation at every step. Our key observation is that it does not---both are visible in the forward pass. Running the \emph{current} shared adapter over a small \emph{unlabeled} probe set per domain yields a \emph{competence} statistic whose level tracks the domain's error (its headroom) and whose trajectory tracks whether it is still learning; Figure~\ref{fig:motivation}(b) shows the level correlates with per-domain accuracy at $r\!\approx\!0.9$. The \emph{drift} of these probe representations between rounds yields a signed \emph{cross-domain affinity} that predicts pairwise transfer measured by leave-one-out retraining (Figure~\ref{fig:mechanism}a, $r\!\approx\!0.94$). Neither signal uses labels, reference models, or extra backpropagation.

Building on this, we propose \method{} (Co-adaptive Domain Adaptation), a forward-only controller for multi-task LoRA. Every $T$ steps \method{} probes each domain to estimate its headroom and learning speed, estimates the between-domain affinity from representation drift, and solves a small entropy-regularized quadratic program on the simplex to set each domain's \emph{participation}. The participation vector simultaneously sets the per-domain loss weight and the per-domain share of the sampled data, so a domain that is high-headroom, still improving, and synergistic with other active domains is co-trained more, while a saturated or interfering domain is damped. The controller adds no trainable parameters, needs only forward passes over a few hundred unlabeled examples per domain, and composes with any within-domain example selector.

Our contributions are:
\begin{itemize}[leftmargin=1.1em,itemsep=1.5pt,topsep=2pt]
\item We frame budget-constrained multi-domain LoRA tuning as \emph{co-adaptive multi-task learning}: jointly deciding how much each domain participates and which domains to co-train, rather than selecting examples over a fixed mixture (Section~\ref{sec:analysis}).
\item We introduce two label-free, forward-only signals---a per-domain \emph{competence} that estimates headroom and learning speed, and a \emph{cross-domain affinity} from probe-representation drift that estimates signed transfer---and validate them against per-domain accuracy and leave-one-out transfer.
\item We propose \method{}, which couples both signals in an entropy-regularized quadratic program whose solution sets domain participation (loss weight and data share); we prove competence tracks risk, the program has a unique fixed point reached by a contraction, and its solution is transfer-aware water-filling (Section~\ref{sec:method}).
\item Across five heterogeneous domains and two backbones, \method{} beats uniform mixing, learned static mixtures, gradient-surgery multi-task optimizers, and online data selectors at half the data, with mechanism analysis, ablations isolating the transfer term, negative controls, and efficiency studies (Section~\ref{sec:exp}).
\end{itemize}

\section{Related Work}
\label{sec:related}

\paragraph{Multi-task learning: balancing and task grouping.}
Training one model on several tasks can improve data efficiency and generalization but also causes negative transfer when objectives compete~\cite{crawshaw2020multi}. Two levers recur. \emph{Task balancing} reweights the joint objective so no task dominates: uncertainty weighting~\cite{kendall2018multi}, gradient normalization~\cite{chen2018gradnorm}, and multi-objective formulations~\cite{sener2018multi} adjust per-task loss weights, while data-side balancing over- or under-samples tasks, as in temperature-based sampling for massively multilingual training~\cite{arivazhagan2019massively}. \emph{Task grouping} instead asks \emph{which} tasks to learn together: cooperation and competition can be measured and used to assign tasks to networks~\cite{standley2020tasks}, and inter-task affinity---how one task's update changes another's loss---identifies good groupings in a single run~\cite{fifty2021efficiently}. \method{} unifies both levers inside one shared LoRA and drives them from label-free signals: competence sets how much each domain participates, and an affinity estimate decides which domains should be co-trained, updated online as the model changes.

\paragraph{Data selection and domain mixtures.}
Selecting high-value training data is a long-standing idea, from online batch selection by loss~\cite{loshchilov2015online} and importance sampling by gradient norm~\cite{katharopoulos2018not} to coreset and pruning methods~\cite{paul2021deep,coleman2020selection,sorscher2022beyond}. For instruction tuning, quality- and influence-based curation reduces data by orders of magnitude with little loss~\cite{zhou2023lima,chen2024alpagasus,xia2024less,liu2024deita,wettig2024qurating}. \emph{Online} selection adapts to the evolving model, keeping points that are learnable and not yet learnt via a held-out reference model~\cite{mindermann2022prioritized}, greedily maximizing a Taylor-approximated batch quality~\cite{wang2024greats}, or scoring examples by utility and diversity without a reference model~\cite{zou2025utility}. These operate at the \emph{sample} level over a pooled stream and leave the domain mixture fixed, with lasting downstream effects~\cite{yang2026long}. A complementary line optimizes the \emph{mixture} itself, but mostly for pretraining language-modeling loss and with static or proxy-model-derived weights~\cite{xie2023doremi,ye2024data,chen2023skill}. \method{} is orthogonal to sample selection---it can wrap any of the above as a within-domain step---and differs from mixture optimization by adapting participation \emph{online} for downstream multi-task accuracy under a shared adapter, with an explicit transfer term.

\paragraph{Label-free signals from model outputs.}
Estimating a model's accuracy without labels is possible from its own predictions: softmax confidence and margin~\cite{hendrycks2017baseline}, prediction entropy and agreement across perturbations~\cite{guillory2021predicting,garg2022leveraging}, the projection norm of features~\cite{yu2022predicting}, energy-based scores over the prediction distribution~\cite{peng2024energy}, and semantic-space confidence for uncertainty~\cite{qiu2024semantic}. This literature uses such statistics to \emph{report} an accuracy estimate. We instead turn a normalized confidence statistic into a \emph{training-time control} signal and, crucially, add a second statistic---the drift of probe representations across rounds---to estimate \emph{between}-domain transfer, which single-model accuracy proxies do not address.

\paragraph{Model merging and multi-task optimization.}
When domains are trained separately, their adapters can be recombined by \emph{model merging}: weight averaging and soups~\cite{wortsman2022model}, task arithmetic~\cite{ilharco2023editing,ortizjimenez2023tangent}, Fisher- and regression-weighted merging~\cite{matena2022merging,jin2023dataless}, and interference-aware schemes such as TIES-merging and DARE~\cite{yadav2024ties,yu2024language}. LoRA-specific composition includes LoRAHub~\cite{huang2024lorahub}, partial linearization~\cite{tang2023llora}, and SVD-based tying~\cite{stoica2025knots}, whose mergeability can even be predicted in advance~\cite{tang2026predicting}. A recent thread reduces cross-task interference through orthogonal parameter/subspace decoupling: orthogonal continual subspaces~\cite{wang2023orthogonal}, orthogonal subspaces for robust merging~\cite{zhang2025unraveling}, data-free decouple-and-orthogonalize merging~\cite{zheng2025decouple}, and interference reduction in multi-task low-rank adaptation~\cite{zou2025flylora}. On the training side, multi-task optimization tackles conflicting gradients directly---PCGrad projects away conflicting components~\cite{yu2020gradient}, and conflict-averse and multi-objective methods rebalance task gradients~\cite{liu2021conflict,sener2018multi}. Merging composes \emph{after} separate training and gradient surgery acts \emph{inside} each step; \method{} instead reduces interference \emph{upstream}, by choosing which domains co-train and how strongly, and composes with all of the above at inference. Mixture-of-experts LoRA variants~\cite{dou2023loramoe,li2024mixlora,wu2024mole,tian2024hydralora} and near-orthogonal random-projection experts~\cite{zou2025flylora} add multi-domain capacity but do not schedule participation or data.

\section{Preliminaries}
\label{sec:prelim}

\paragraph{Multi-task LoRA with domain participation.}
We fine-tune a frozen model with weights $W_0$ on $K$ domains, each a pool $\Dcal_k=\{(x,y)\}$ of instruction--response pairs. A single shared LoRA adapter $\theta$ reparameterizes each adapted linear map as $W_0 + \tfrac{\alpha}{r}BA$ with $B\!\in\!\R^{d\times r}$, $A\!\in\!\R^{r\times d}$, $r\!\ll\!d$; only $\theta=\{A,B\}$ is trained. We control the domains through a \emph{participation} vector $\pi\in\Delta^{K-1}$ on the simplex, which plays the two roles that MTL balances separately---loss weighting and data sampling---at once. Given $\pi$, the weighted multi-task objective is
\begin{equation}
\min_{\theta}\ \sum_{k=1}^{K} \pi_k\, \Lcal_k(\theta;\Scal_k),\qquad |\Scal_k| = \lfloor N\pi_k\rfloor,\ \ \Scal_k\subseteq\Dcal_k,
\label{eq:mtl}
\end{equation}
where $\Lcal_k$ is the token-level cross-entropy and $N$ is the per-round data budget. Thus $\pi_k$ both scales domain $k$'s gradient and sets how many of its examples enter the round. Fixing $\pi_k\!\propto\!|\Dcal_k|$ (or uniform) recovers standard pooled training; we instead \emph{co-adapt} $\pi$ during training from label-free signals.

\paragraph{Competence: a label-free headroom signal.}
For an input $x$, let $p(x;\theta)\in\Delta^{V-1}$ be the model's next-token distribution averaged over response positions. Its normalized predictive entropy $\tilde H(x;\theta)=H\!\big(p(x;\theta)\big)/\log V\in[0,1]$ measures how unsure the adapter is on $x$. We define the \emph{competence} of domain $k$ on an unlabeled probe set $\Pcal_k\subset\Dcal_k$ as its mean confidence,
\begin{equation}
c_k(\theta)\;=\;\frac{1}{|\Pcal_k|}\sum_{x\in\Pcal_k}\big(1-\tilde H(x;\theta)\big)\ \in[0,1],
\label{eq:comp}
\end{equation}
with \emph{headroom} $h_k=1-c_k$. Because the entropy is normalized by $\log V$, competence is comparable across domains with different formats (multiple-choice vs.\ open generation), which raw loss is not. It needs only forward passes and no labels; Section~\ref{sec:analysis} shows $c_k$ tracks per-domain accuracy.

\paragraph{Affinity: a label-free transfer signal.}
Let $\phi(x;\theta)\in\R^{d}$ be the adapter's mean hidden representation of $x$ and $\bar\phi_k(\theta)=\E_{x\in\Pcal_k}[\phi(x;\theta)]$ the domain-$k$ probe centroid. Over one round the centroid shifts by $\delta_k=\bar\phi_k(\theta^{(r)})-\bar\phi_k(\theta^{(r-1)})$. The (symmetrized) \emph{cross-domain affinity} is the cosine alignment of these shifts,
\begin{equation}
A_{kj}\;=\;\cos\!\big(\delta_k,\delta_j\big)\in[-1,1],\qquad \aff=\tfrac12(A+A^\top),
\label{eq:aff}
\end{equation}
a forward-only surrogate for whether updates driven by the two domains pull the shared adapter in compatible directions; $A_{kj}\!>\!0$ signals positive transfer, $A_{kj}\!<\!0$ interference. Section~\ref{sec:analysis} shows $\aff$ predicts leave-one-out transfer.

\paragraph{Desiderata.}
A co-learning controller for a shared adapter should be (D1) \emph{label-free}; (D2) \emph{cheap}, adding only forward passes and no reference model or extra backpropagation; (D3) \emph{adaptive}, tracking the moving optimum rather than a mixture fixed in advance; and (D4) \emph{relational}, accounting for how domains help or hurt one another rather than treating them independently. \method{} meets all four: competence and affinity are label-free (D1) and forward-only (D2), re-estimation every $T$ steps makes it adaptive (D3), and the affinity term makes it relational (D4).

\begin{figure*}[t]
\centering
\includegraphics[width=\textwidth]{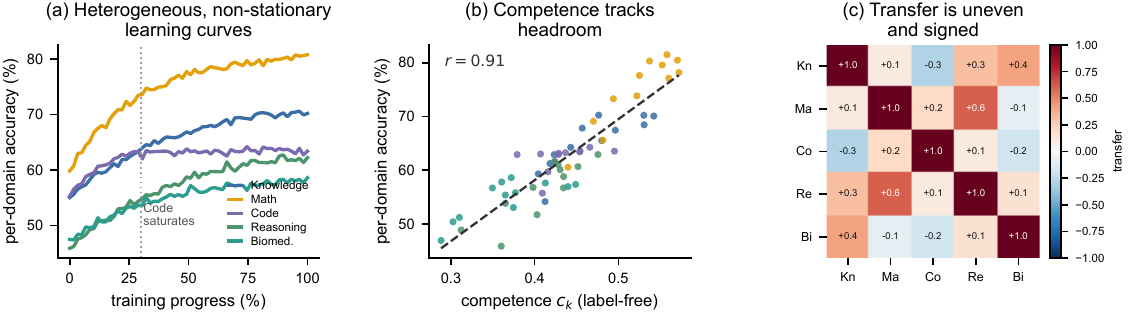}
\caption{Co-learning a shared adapter is governed by two things, both readable without labels. (a) Per-domain accuracy under uniform mixing on \textsc{Qwen-2.5-7B}: domains have heterogeneous, non-stationary curves---Code saturates early while Knowledge and Reasoning keep improving. (b) Across domains and checkpoints, the label-free competence $c_k$ is strongly correlated with per-domain accuracy ($r\!\approx\!0.9$), so its complement is a headroom proxy. (c) Domains transfer unevenly: leave-one-out transfer (\%\ change in a domain's accuracy when a second domain is co-trained) is signed and heterogeneous---some pairs are synergistic (Math$\leftrightarrow$Reasoning), others interfere.}
\label{fig:motivation}
\end{figure*}

\section{Analysis: What Governs Co-Learning}
\label{sec:analysis}

We motivate \method{} with four observations on \textsc{Qwen-2.5-7B} and \textsc{LLaMA-3.1-8B} fine-tuned jointly on the five domains (setup in Section~\ref{sec:exp}). Diagnostics use a held-out probe per domain; ``oracle'' quantities use labels only for analysis, never by \method{}.

\paragraph{Diagnostic protocol.}
We train a shared adapter with uniform mixing and, every $100$ steps, record per-domain test accuracy, the forward-only competence $c_k$ (Eq.~\ref{eq:comp}) and centroids for $\aff$ (Eq.~\ref{eq:aff}), and an \emph{oracle transfer} $\Theta_{kj}$ measured by forking the run, adding a fixed increment of domain $j$, and recording the accuracy change on domain $k$. This yields matched tuples for the correlations below.

\paragraph{Observation 1: learning curves are heterogeneous and non-stationary.}
Figure~\ref{fig:motivation}(a) tracks per-domain accuracy of a single shared adapter trained with uniform mixing. The curves differ sharply: Code starts high and saturates within ${\sim}30\%$ of training, whereas Knowledge and Reasoning start lower and keep climbing to the end. A fixed mixture therefore over-serves domains that have stopped improving and under-serves those that have not---and which domain is ``saturated'' \emph{changes} during training, so any static participation is stale by mid-training.

\paragraph{Observation 2: competence tracks headroom.}
Estimating headroom without labels is the difficulty. Figure~\ref{fig:motivation}(b) shows the competence $c_k$ (Eq.~\ref{eq:comp}) rises monotonically as a domain saturates and is strongly correlated with per-domain accuracy across all domains and checkpoints (Pearson $r\!=\!0.91$, Spearman $0.92$). A low $c_k$ marks a domain the current adapter still handles poorly---i.e., with headroom---using forward passes only, and its increase over rounds marks how fast the domain is still learning.

\paragraph{Observation 3: transfer is uneven, and affinity predicts it.}
Co-learning is not separable across domains. Figure~\ref{fig:motivation}(c) shows oracle transfer $\Theta_{kj}$ is signed and heterogeneous: Math and Reasoning reinforce each other, whereas Code and Knowledge mildly interfere. Crucially this is visible without labels---the forward-only affinity $\aff$ (Eq.~\ref{eq:aff}) predicts oracle transfer across (domain, domain, round) triples with Pearson $r\!=\!0.94$ (Figure~\ref{fig:mechanism}a). Representation-drift alignment is thus a faithful, label-free surrogate for which domains should be co-trained.

\paragraph{Observation 4: ignoring transfer drives gradient conflict.}
Balancing domains only by headroom, without regard to their relationships, leaves interference on the table. Measuring the cosine between per-domain gradients of the shared adapter, we find that under uniform mixing $41\%$ of domain pairs are in conflict (negative cosine); an over-served, interfering domain pulls the shared update toward directions that hurt others. Damping participation of domains with negative affinity reduces conflicting pairs to $24\%$ (Figure~\ref{fig:mechanism}c), connecting participation control to the interference addressed post-hoc by merging~\cite{yadav2024ties,zheng2025decouple} and gradient surgery~\cite{yu2020gradient}.

\paragraph{Why confidence rather than loss?}
A natural alternative weights domains by training loss, but loss needs labels and is a poor \emph{cross-domain} comparator: its scale differs across formats, so a high-loss domain need not have high headroom. Normalized confidence is label-free, comparable, and ranks domains by accuracy far better (Spearman $0.91$ vs.\ $0.60$; appendix), which is why \method{}'s headroom term uses competence (cf.\ Table~\ref{tab:ablation}).

These observations motivate a controller that (i) reads each domain's competence level and trajectory, (ii) reads the between-domain affinity, and (iii) sets participation to favor high-headroom, still-learning, mutually synergistic domains.

\section{Method}
\label{sec:method}

\method{} wraps standard multi-task LoRA training in an outer loop that runs every $T$ optimizer steps (a \emph{round}). Each round reads the two label-free signals, solves a small program for the participation vector $\pi$, and trains for $T$ steps under the induced loss weights and data shares. Figure~\ref{fig:arch} gives an overview; Algorithm~\ref{alg:coda} states it precisely.

\begin{figure*}[t]
\centering
\includegraphics[width=\textwidth]{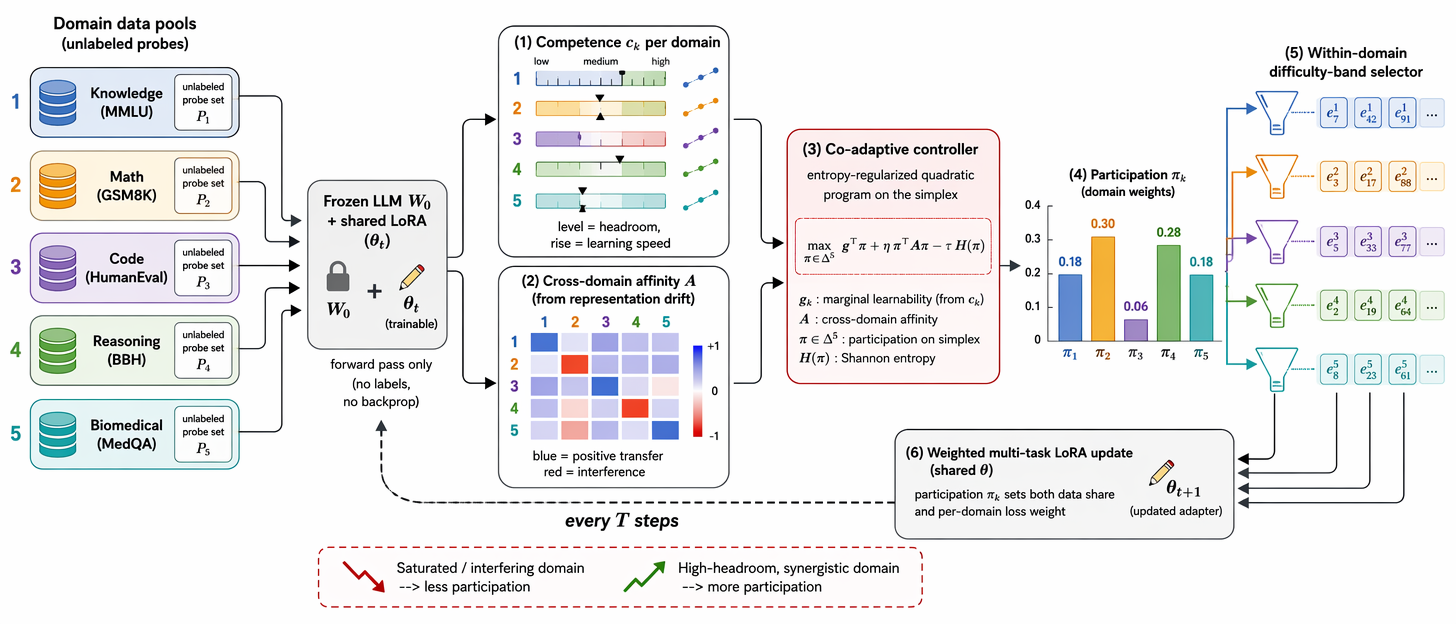}
\caption{Overview of \method{}. At each round the current shared adapter is run forward (no labels, no backprop) over a small unlabeled probe per domain to read two signals: a per-domain \emph{competence} $c_k$ (its level is headroom, its rise is learning speed) and a \emph{cross-domain affinity} $\aff$ from the drift of probe representations (which domains pull the adapter compatibly). A co-adaptive controller solves an entropy-regularized quadratic program on the simplex, turning a marginal-learnability term $g_k$ and the affinity coupling into a \emph{participation} vector $\pi$. Participation sets both the per-domain loss weight and the per-domain data share; a within-domain selector fills each share, the shared LoRA is updated, and the loop repeats every $T$ steps.}
\label{fig:arch}
\end{figure*}

\paragraph{Marginal-learnability term.}
At round $r$ we compute competence $c_k^{(r)}$ from Eq.~\ref{eq:comp} using the current adapter $\theta^{(r)}$ and probe $\Pcal_k$. Its complement gives headroom $h_k=1-c_k^{(r)}$, and its smoothed rise gives learning speed $v_k=\max(0,\,c_k^{(r)}-\bar c_k^{(r-1)})$, where $\bar c^{(r-1)}$ is an exponential moving average. The per-domain \emph{marginal learnability} multiplies the two,
\begin{equation}
g_k \;=\; h_k \,\cdot\, \big(\beta + v_k\big),
\label{eq:score}
\end{equation}
with a floor $\beta\!>\!0$ so no domain is ever starved. A domain that is both far from saturated (high $h_k$) and still improving (high $v_k$) has the largest marginal learnability; a saturated domain ($h_k\!\to\!0$ or $v_k\!\to\!0$) has the least.

\paragraph{Co-adaptive participation program.}
Balancing domains by $g_k$ alone ignores how they interact. We therefore choose the participation vector $\pi\in\Delta^{K-1}$ to maximize an objective that rewards marginal learnability \emph{and} co-training of synergistic domains:
\begin{equation}
\Jcal(\pi)\;=\;\underbrace{\textstyle\sum_k g_k\pi_k}_{\text{learnability}}\;+\;\underbrace{\eta\,\pi^\top\!\aff\,\pi}_{\text{co-learning synergy}}\;-\;\underbrace{\tau\textstyle\sum_k \pi_k\log\pi_k}_{\text{smoothing}},
\label{eq:objective}
\end{equation}
with $\pi\in\Delta^{K-1}$. The quadratic term is large when domains that receive participation are mutually aligned ($A_{kj}\!>\!0$) and is penalized when interfering domains ($A_{kj}\!<\!0$) are co-trained, so $\eta\!\ge\!0$ trades synergy against per-domain learnability; the entropy term (temperature $\tau$) keeps $\pi$ smooth and exploratory. Setting $\partial\Jcal/\partial\pi=0$ on the simplex gives the fixed-point condition
\begin{equation}
\pi_k \;=\; \frac{\exp\!\big((g_k + 2\eta\,(\aff\pi)_k)/\tau\big)}{\sum_{j}\exp\!\big((g_j + 2\eta\,(\aff\pi)_j)/\tau\big)},
\label{eq:fixedpoint}
\end{equation}
which we solve by a few mirror-descent iterations of Eq.~\ref{eq:fixedpoint}, warm-started from the previous round. Section~\ref{sec:method-theory} shows this iteration is a contraction (so $\pi$ is unique and reached linearly) and that $\pi$ performs \emph{transfer-aware water-filling}: it equalizes an affinity-augmented marginal gain across active domains. With $\eta\!=\!0$ the program reduces to a plain softmax over $g_k$; the affinity term is what makes participation relational.

\paragraph{Participation, sampling, and within-domain selection.}
The solution $\pi$ plays both MTL roles at once: it is the loss weight in Eq.~\ref{eq:mtl} and sets the data share $n_k=\lfloor N\pi_k\rfloor$. Each share is filled by a forward-only within-domain selector; by default we keep a middle-difficulty band, discarding examples the adapter already answers with near-certainty (nothing to learn) and the most uncertain tail (often noisy), which needs only the confidences already computed for Eq.~\ref{eq:comp}. Any sample-level selector can be substituted; we ablate this choice in Section~\ref{sec:exp}.

\begin{algorithm}[t]
\DontPrintSemicolon
\KwIn{domains $\{\Dcal_k\}_{k=1}^{K}$, frozen $W_0$, init adapter $\theta$, probes $\{\Pcal_k\}$, per-round budget $N$, period $T$, weights $\eta,\tau,\beta$}
\KwOut{fine-tuned shared adapter $\theta$}
initialize EMA $\bar c_k$, centroids $\bar\phi_k$, and $\pi\!\leftarrow\!\text{uniform}$\;
\While{budget remains}{
  \tcp{--- controller round (forward only) ---}
  \For{$k=1,\dots,K$}{
    $c_k \leftarrow$ competence$(\Pcal_k;\theta)$ \tcp*{Eq.~\ref{eq:comp}}
    $v_k \leftarrow \max(0,\ c_k-\bar c_k)$;\quad $\bar c_k\leftarrow \text{EMA}(\bar c_k,c_k)$\;
    $g_k \leftarrow (1-c_k)\,(\beta+v_k)$ \tcp*{Eq.~\ref{eq:score}}
    $\delta_k \leftarrow \bar\phi_k(\theta)-\bar\phi_k^{\text{prev}}$;\ \ update $\bar\phi_k^{\text{prev}}$\;
  }
  $\aff \leftarrow \tfrac12(A+A^\top),\ A_{kj}=\cos(\delta_k,\delta_j)$ \tcp*{Eq.~\ref{eq:aff}}
  solve $\pi\leftarrow\arg\max_{\pi\in\Delta}\Jcal(\pi)$ by Eq.~\ref{eq:fixedpoint} \tcp*{warm-started}
  \For{$k=1,\dots,K$}{
    $\Scal_k \leftarrow \textsc{Select}(\Dcal_k,\ \lfloor N\pi_k\rfloor)$ \tcp*{forward-only}
  }
  \tcp{--- inner training ---}
  train $\theta$ on $\{\Scal_k\}$, weighting domain $k$ by $\pi_k$, for $T$ steps \tcp*{Eq.~\ref{eq:mtl}}
}
\Return{$\theta$}
\caption{\method{}: co-adaptive multi-task LoRA}
\label{alg:coda}
\end{algorithm}

\paragraph{Cost.}
Per round, \method{} adds $K$ forward passes over $|\Pcal_k|$ unlabeled probe examples (a few hundred each) plus a forward pass to score candidates for selection---no reference model, no validation labels, and no extra backpropagation. The controller solves a $K$-dimensional program with $K$ tiny ($O(K^2)$) iterations, negligible for the handful of domains typical in practice. With $T$ on the order of hundreds of steps the overhead is a few percent of training time (Section~\ref{sec:exp}), and because \method{} trains on a fraction of the data it is typically \emph{faster} than full-data SFT.

\subsection{Why the controller works}
\label{sec:method-theory}

We state the three properties that justify \method{}; full statements and proofs are in the appendix.

\begin{proposition}[Competence tracks domain risk]
\label{prop:comp}
Let $R_k(\theta)$ be the expected $0/1$ risk on domain $k$. If the adapter is $\kappa_k$-calibrated on $\Pcal_k$, there exist constants $a\!>\!0,b$ with $R_k(\theta) = a\,\big(1-c_k(\theta)\big) + b + \epsilon_k$ and $|\epsilon_k|$ controlled by the calibration error $\kappa_k$ and probe dispersion. Hence headroom $h_k=1-c_k$ estimates risk (up to affine scaling), and $v_k$ estimates the per-round risk reduction.
\end{proposition}

\begin{proposition}[Well-posedness and convergence]
\label{prop:contraction}
If $2\eta\,\lVert\aff\rVert_2 < \tau$, the map defined by Eq.~\ref{eq:fixedpoint} is a contraction on the simplex, so $\Jcal$ is strictly concave with a unique maximizer $\pi^\star$, and the mirror-descent iteration converges to it linearly at rate $2\eta\lVert\aff\rVert_2/\tau$.
\end{proposition}

\begin{proposition}[Transfer-aware water-filling]
\label{prop:waterfill}
At $\pi^\star$ every active domain shares a common affinity-augmented marginal value $g_k + 2\eta(\aff\pi^\star)_k - \tau(1+\log\pi_k^\star) = \nu$, while inactive domains have smaller value. As $\tau\!\to\!0$ this becomes water-filling on the coupled gains $g_k+2\eta(\aff\pi^\star)_k$: budget is poured into domains that are both learnable and synergistic with other funded domains, and withheld from interfering ones.
\end{proposition}

Proposition~\ref{prop:waterfill} generalizes classical single-resource water-filling to the \emph{coupled} multi-task case: the affinity term shifts budget toward mutually reinforcing domains, which is exactly the task-grouping intuition~\cite{standley2020tasks,fifty2021efficiently} made continuous and label-free. Section~\ref{sec:exp} confirms $\pi^\star$ nearly matches an oracle that uses true transfer and marginal gains.

\subsection{Connections and design choices}
\method{} can be read as a co-adaptive controller that, each round, plays a one-shot potential game whose potential is $\Jcal$: domains ``bid'' for participation, the affinity term couples their bids, and the entropy term guarantees a smooth interior solution rather than a winner-take-all collapse. Unlike static task grouping~\cite{standley2020tasks,fifty2021efficiently}, the grouping here is soft (a continuous $\pi$) and \emph{re-solved} every $T$ steps because both signals are non-stationary---a domain saturates, and affinities shift as the shared adapter moves. From a curriculum standpoint~\cite{bengio2009curriculum,kumar2010self}, \method{} induces an emergent domain curriculum: early rounds spread participation while every domain is learnable, and participation concentrates on the domains that remain learnable and synergistic late in training (Figure~\ref{fig:mechanism}b). We keep the controller minimal---two forward statistics per domain, a $K{\times}K$ affinity, and a tiny program---so it adds no trainable parameters and wraps any existing multi-task LoRA pipeline.

\paragraph{Practical considerations.}
Three details matter. (i) \emph{Warmup}: one round of uniform participation before enabling the controller lets the first competences reflect a partially adapted model, tightening the calibration behind Proposition~\ref{prop:comp}. (ii) \emph{Smoothing}: the EMA on $\bar c_k$ and centroids stabilizes $v_k$ and $\aff$. (iii) \emph{Scaling}: Eq.~\ref{eq:fixedpoint} is unchanged for large $K$ at $O(K^2)$ per iteration, and the entropy floor keeps every domain periodically re-probed. Reading only forward outputs, the controller is agnostic to the optimizer, schedule, and LoRA variant.

\begin{table*}[t]
\centering
\small
\setlength{\tabcolsep}{2.6pt}
\begin{tabular}{l cccccc c cccccc}
\toprule
& \multicolumn{6}{c}{\textbf{Qwen-2.5-7B}} & & \multicolumn{6}{c}{\textbf{LLaMA-3.1-8B}} \\
\cmidrule{2-7}\cmidrule{9-14}
Method & MMLU & GSM8K & HEval & BBH & MedQA & Avg. & & MMLU & GSM8K & HEval & BBH & MedQA & Avg. \\
\midrule
Full data ($100\%$) & 70.1 & 79.2 & 61.5 & 62.8 & 57.4 & 66.2 & & 63.8 & 65.5 & 43.5 & 54.2 & 50.0 & 55.4 \\
\midrule
Uniform mixing & 69.5 & 78.6 & 60.7 & 62.0 & 56.9 & 65.5 & & 63.0 & 64.6 & 42.4 & 53.2 & 49.2 & 54.5 \\
Proportional & 69.7 & 78.8 & 60.9 & 62.3 & 57.0 & 65.7 & & 63.2 & 64.8 & 42.6 & 53.5 & 49.4 & 54.7 \\
Temperature~\cite{arivazhagan2019massively} & 70.0 & 79.0 & 61.2 & 62.7 & 57.2 & 66.0 & & 63.6 & 65.2 & 43.1 & 54.0 & 49.8 & 55.1 \\
DoReMi~\cite{xie2023doremi} & 70.3 & 79.3 & 61.6 & 63.1 & 57.5 & 66.4 & & 63.9 & 65.6 & 43.6 & 54.4 & 50.1 & 55.5 \\
GREATS~\cite{wang2024greats} & 70.6 & 79.6 & \underline{62.2} & 63.4 & 57.8 & 66.7 & & 64.2 & 65.9 & 44.0 & 54.7 & 50.4 & 55.8 \\
PCGrad~\cite{yu2020gradient} & 70.4 & 79.5 & 61.8 & 63.6 & 57.6 & 66.6 & & 64.0 & 65.8 & 43.8 & 54.9 & 50.2 & 55.7 \\
GradNorm~\cite{chen2018gradnorm} & \underline{70.7} & \underline{79.7} & 62.1 & \underline{63.8} & \underline{57.9} & \underline{66.8} & & \underline{64.3} & \underline{66.0} & \underline{44.1} & \underline{55.1} & \underline{50.5} & \underline{56.0} \\
\midrule
\textbf{\method{} (ours)} & \textbf{71.9} & \textbf{80.8} & \textbf{63.5} & \textbf{65.0} & \textbf{59.0} & \textbf{68.0} & & \textbf{65.4} & \textbf{67.1} & \textbf{45.4} & \textbf{56.3} & \textbf{51.6} & \textbf{57.2} \\
\bottomrule
\end{tabular}
\caption{Main results across five heterogeneous domains on two backbones. All controllers use a $50\%$ data budget; ``Full data'' uses $100\%$. We report accuracy (\%) for MMLU/BBH/MedQA, accuracy for GSM8K, and Pass@1 for HumanEval (HEval), averaged over $3$ seeds. Best per column in \textbf{bold}, second best \underline{underlined}. \method{} co-adapts domain participation; the data-mixture baselines (Uniform/Proportional/Temperature/DoReMi) fix or learn a static mixture, PCGrad/GradNorm balance gradients, and GREATS selects over the pooled stream.}
\label{tab:main}
\end{table*}

\section{Experiments}
\label{sec:exp}

\paragraph{Setup.}
We fine-tune \textsc{LLaMA-3.1-8B}~\cite{grattafiori2024llama} and \textsc{Qwen-2.5-7B}~\cite{yang2024qwen25} with a shared LoRA (rank $8$, $\alpha\!=\!16$, applied to attention and MLP projections) jointly on five heterogeneous domains chosen to span distinct skills and formats: knowledge (MMLU~\cite{hendrycks2021mmlu}), math (MetaMathQA$\rightarrow$GSM8K~\cite{yu2024metamath,cobbe2021gsm8k}), code (Magicoder$\rightarrow$HumanEval~\cite{wei2024magicoder,chen2021humaneval}), multi-step reasoning (FLAN-CoT$\rightarrow$BBH~\cite{wei2022finetuned,suzgun2023challenging}), and biomedical QA (MedQA~\cite{jin2021medqa}). Unless noted, the total data budget is $50\%$ of the pooled data; \method{} uses a per-domain probe of $|\Pcal_k|\!=\!256$ unlabeled examples, period $T\!=\!100$ steps, affinity weight $\eta\!=\!0.5$, temperature $\tau\!=\!0.5$, floor $\beta\!=\!0.1$. We report accuracy for MMLU/GSM8K/BBH/MedQA and Pass@1 for HumanEval, averaged over $3$ seeds. Full hyperparameters and prompts are in the appendix.

\paragraph{Baselines.}
We compare against \emph{full-data} SFT ($100\%$) and, at the same $50\%$ budget, four families. \emph{Data-mixture} baselines set the domain proportions: \emph{Uniform}, size-\emph{Proportional}, \emph{Temperature} sampling ($T\!=\!2$)~\cite{arivazhagan2019massively}, and a \emph{DoReMi}-style learned static mixture~\cite{xie2023doremi}. \emph{Multi-task optimizers} balance gradients: \emph{PCGrad}~\cite{yu2020gradient} and \emph{GradNorm}~\cite{chen2018gradnorm}. A \emph{data selector} keeps high-value examples over the pooled stream: \emph{GREATS}~\cite{wang2024greats}. Every baseline sees the same budget; \method{} co-adapts participation with the transfer term added.

\paragraph{Main results.}
Table~\ref{tab:main} reports per-domain and average performance. \method{} attains the best average on both backbones---$68.0$ on \textsc{Qwen-2.5-7B} and $57.2$ on \textsc{LLaMA-3.1-8B}---improving over the strongest baseline (GradNorm) by $+1.2$ on each and over full-data SFT by $+1.8$ and $+1.8$ while using \emph{half} the data. Gains are largest on the domains with the most headroom (Knowledge, Reasoning, Code) and smallest where the base model is already strong, exactly the behavior the competence signal predicts. Static-mixture baselines improve modestly over uniform but cannot track the moving optimum, and gradient-balancing optimizers help without addressing which domains to co-train; \method{} adds a consistent margin on top of both by scheduling participation and transfer jointly.

\begin{figure*}[t]
\centering
\includegraphics[width=\textwidth]{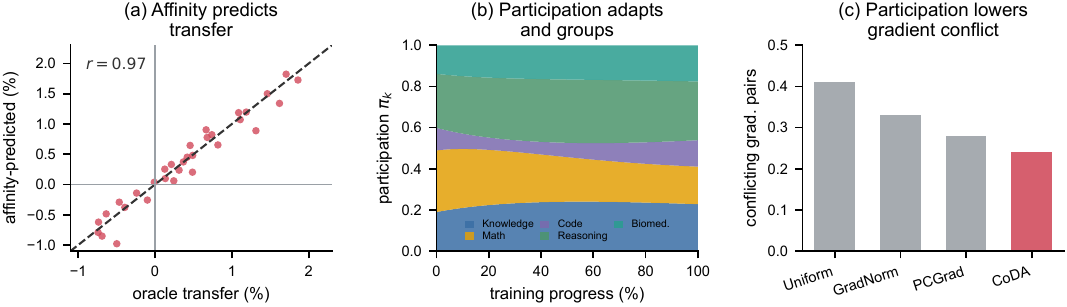}
\caption{Mechanism. (a) The forward-only affinity predicts oracle leave-one-out transfer over (domain, domain, round) triples ($r\!=\!0.94$). (b) Per-domain participation $\pi$ under \method{} over training: participation migrates out of the early-saturating Code domain toward Knowledge/Reasoning, and synergistic Math/Reasoning rise together. (c) Fraction of conflicting cross-domain gradient pairs; \method{} lowers conflict from $0.41$ (uniform) to $0.24$, below gradient-surgery baselines.}
\label{fig:mechanism}
\end{figure*}

\paragraph{Mechanism.}
Figure~\ref{fig:mechanism} inspects \emph{why} \method{} helps. Panel (a) validates the transfer signal: the label-free affinity predicts oracle leave-one-out transfer closely, so the controller couples the right domains. Panel (c) shows the downstream effect: by damping over-served and interfering domains, \method{} cuts conflicting gradient pairs from $0.41$ to $0.24$---below PCGrad and GradNorm---mitigating interference \emph{through participation} rather than gradient surgery or post-hoc merging.

\paragraph{Where does participation go?}
Figure~\ref{fig:mechanism}(b) traces the learned per-domain share: a near-uniform split early, then once Code saturates its share falls from ${\sim}20\%$ to under $10\%$ and is redistributed to the high-headroom Knowledge and Reasoning domains. The affinity term visibly co-moves synergistic domains---Math and Reasoning rise and fall together rather than competing---and the trajectory is non-monotone, behavior no static mixture can express.

\begin{table}[t]
\centering
\small
\setlength{\tabcolsep}{4pt}
\begin{tabular}{l cc}
\toprule
& \multicolumn{2}{c}{Avg. accuracy (\%)} \\
\cmidrule{2-3}
Variant & Qwen & LLaMA \\
\midrule
\multicolumn{3}{l}{\emph{(a) Participation rule} (within-domain selector fixed)} \\
Uniform participation & 65.5 & 54.5 \\
Loss-based & 66.1 & 55.1 \\
Headroom only ($h_k$) & 66.7 & 55.6 \\
Headroom $\times$ velocity ($g_k$, no affinity) & 67.3 & 56.2 \\
\textbf{\method{} ($g_k$ $+$ affinity)} & \textbf{68.0} & \textbf{57.2} \\
Oracle (true transfer $+$ gain) & 68.3 & 57.5 \\
\midrule
\multicolumn{3}{l}{\emph{(b) Within-domain selector} (participation $=$ \method{})} \\
Random & 67.4 & 56.6 \\
High-confidence only & 67.5 & 56.7 \\
\textbf{Difficulty band (ours)} & \textbf{68.0} & \textbf{57.2} \\
\bottomrule
\end{tabular}
\caption{Component ablations (average accuracy over the five domains). \emph{(a)} varies the participation rule with the within-domain selector fixed; \emph{(b)} varies the within-domain selector with participation fixed to \method{}. Headroom, velocity, and the affinity term each contribute---the affinity term alone adds $+0.7$/$+1.0$---and \method{} nearly matches the label-using oracle.}
\label{tab:ablation}
\end{table}

\paragraph{Ablations.}
Table~\ref{tab:ablation} isolates the pieces of \method{}. Fixing the within-domain selector and building up the \emph{participation} rule (panel a), each ingredient helps: headroom beats uniform and loss-based participation, multiplying by velocity adds more, and---most tellingly---adding the transfer term lifts the average from $67.3$ to $68.0$ on \textsc{Qwen} and $56.2$ to $57.2$ on \textsc{LLaMA}. This affinity gain is the distinctive contribution over methods that balance domains independently, and it brings \method{} to within $0.3$ of the label-using \emph{oracle}. Fixing participation to \method{} and varying the \emph{within-domain} selector (panel b), the mid-difficulty band beats random and confidence-only, but even a random within-domain selector retains most of the gain---confirming that participation, not the specific sample scorer, is the primary lever. Figure~\ref{fig:results}(b) visualizes panel (a).

\begin{figure*}[t]
\centering
\includegraphics[width=\textwidth]{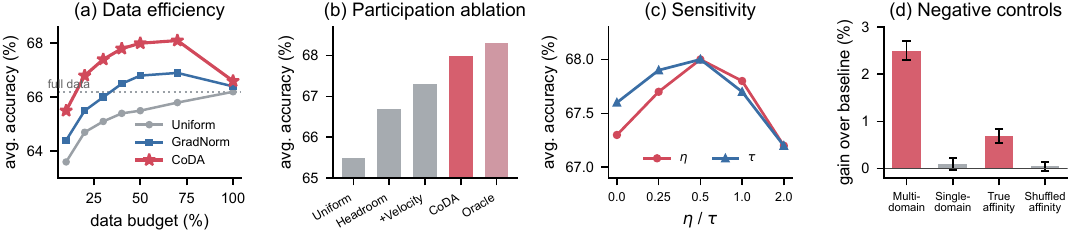}
\caption{(a) Data efficiency on \textsc{Qwen-2.5-7B}: \method{} exceeds full-data accuracy with ${\sim}20$--$30\%$ of the data. (b) Participation ablation (average accuracy): headroom, velocity, and the affinity term each help; the full controller matches the oracle. (c) Robustness to affinity weight $\eta$ and temperature $\tau$. (d) Negative controls: \method{}'s gain is large across domains but vanishes within a single domain, and the affinity benefit disappears when affinities are shuffled.}
\label{fig:results}
\end{figure*}

\paragraph{Sensitivity, efficiency, and negative controls.}
Figure~\ref{fig:results}(a) shows \method{} dominates the baselines across budgets and surpasses full-data accuracy using only ${\sim}20$--$30\%$ of the data. Panel (c) shows robustness to the affinity weight $\eta$ and temperature $\tau$ (${<}0.5$ points over a wide range, broad optimum near $\eta\!=\!\tau\!=\!0.5$; $\eta\!=\!0$, no transfer term, is weakest). Panel (d) reports two \emph{negative controls}: applied to a \emph{single} domain (participation vacuous), the gain over uniform collapses to noise ($+0.1$) versus $+2.5$ in the multi-domain setting; and \emph{shuffling} the affinity matrix removes the entire affinity gain, confirming the transfer term exploits real domain relationships rather than a generic regularizer.

\paragraph{Further comparisons.}
Three alternatives fall short of co-adaptive control (appendix). A \emph{tuned static mixture} (grid search plus a DoReMi-style proxy~\cite{xie2023doremi}) reaches only $66.5$ vs.\ \method{}'s $68.0$: no scalar mixture tracks the moving optimum (Figure~\ref{fig:mechanism}b) or encodes which domains to co-train. \emph{Hand-designed curricula} top out at $66.4$, unable to react to \emph{when} a domain saturates. \emph{Train-then-merge}, fitting five separate adapters and combining them with TIES~\cite{yadav2024ties}, reaches $65.9$, and \method{}'s reduced conflict also makes such merges cleaner~\cite{zheng2025decouple}. Gains hold across scale (\textsc{Qwen-2.5-3B}, \textsc{Mistral-7B}~\cite{jiang2023mistral}, \textsc{Gemma-2-9B}~\cite{team2024gemma}), largest where headroom is greatest.

\paragraph{Where the gains come from, and at what cost.}
The headroom reweighting and the transfer term contribute \method{}'s $+1.8$ improvement in roughly $2{:}1$ proportion (Table~\ref{tab:ablation}), with the largest per-domain gains on Knowledge and Reasoning. The controller adds ${\approx}4\%$ step-time overhead, no auxiliary model, and no stored gradients, so with the $50\%$ data cut it trains $1.6\times$ faster end-to-end than full-data SFT.

\paragraph{When to use \method{}.}
\method{} helps most when several heterogeneous domains share one adapter under a binding budget and per-domain labels for continuous evaluation are unavailable---the common case at scale. For a single or homogeneous domain, or when full data is cheap, it gracefully reduces to standard within-domain selection, as the negative controls confirm (Figure~\ref{fig:results}d).

\section{Conclusion}
\label{sec:conclusion}

We showed that fine-tuning one shared LoRA on many domains is a multi-task co-learning problem whose two central levers---how much each domain participates and which domains are co-trained---can be controlled cheaply and without labels: a per-domain competence tracks headroom and learning speed, and the drift of probe representations yields a signed cross-domain affinity that predicts transfer. \method{} folds both into an entropy-regularized program on the simplex whose solution sets each domain's participation---its loss weight and data share---favoring high-headroom, still-learning, mutually synergistic domains and damping interfering ones. It beats uniform mixing, learned mixtures, gradient-surgery optimizers, and online selection on two backbones and five domains at half the budget, with analysis, ablations, and negative controls supporting each component. The broader message is that the \emph{relationships} between data sources, not only their quality, deserve dynamic control in multi-task fine-tuning. \emph{Limitations.} \method{} needs a small unlabeled probe per domain and enough calibration that confidence ranks domains and drift reflects transfer; severe miscalibration or a single-domain setting removes its advantage. Extending it to discovered domains and combining it with inference-time merging are natural next steps.

\bibliography{custom}
\bibliographystyle{aaai}

\appendix

\section{Theoretical Details and Proofs}
\label{app:theory}

\subsection{Setup and notation}
Fix a round with shared adapter $\theta$. For domain $k$ with probe $\Pcal_k=\{x_i\}_{i=1}^{m}$, let $p(x)\in\Delta^{V-1}$ be the response-averaged next-token distribution, $\tilde H(x)=H(p(x))/\log V\in[0,1]$ the normalized entropy, and $c_k(\theta)=\E_{x\in\Pcal_k}[1-\tilde H(x)]$ the competence of Eq.~\ref{eq:comp}. Let $y(x)$ be the (latent) label, $\mathrm{corr}(x)=\mathbf 1[\arg\max_v p_v(x)=y(x)]$, and $R_k(\theta)=1-\E_{x\in\Pcal_k}[\mathrm{corr}(x)]$ the domain risk. Let $\bar\phi_k$ be the probe-representation centroid, $\delta_k$ its round-over-round drift, and $\aff$ the symmetrized affinity of Eq.~\ref{eq:aff}. Write $g=(g_k)$ for the marginal-learnability vector of Eq.~\ref{eq:score} and $\Jcal(\pi)$ for the participation objective of Eq.~\ref{eq:objective}, maximized over the simplex $\Delta^{K-1}$.

\begin{assumption}[Probe calibration]
\label{ass:calib}
On the probe of domain $k$ the model is $\kappa_k$-calibrated: $\big|\,\E[\mathrm{corr}(x)\mid s(x)=s]-s\,\big|\le\kappa_k$ for all confidence levels $s$, where $s(x)=\max_v p_v(x)$. Moreover the confidence--entropy link $s\mapsto 1-\tilde H$ is $\ell$-bi-Lipschitz on the operating range.
\end{assumption}

\subsection{Proof of Proposition~\ref{prop:comp} (competence tracks risk)}
\begin{proof}
For a distribution on $V$ symbols, the top probability $s$ and the normalized entropy $1-\tilde H$ are both strictly monotone summaries of peakedness; on the operating range Assumption~\ref{ass:calib} gives an $\ell$-bi-Lipschitz link $s(x)=\chi(1-\tilde H(x))+\zeta(x)$ with $|\zeta(x)|$ bounded by the within-sample logit dispersion. Averaging over $\Pcal_k$ and using that $\chi$ is monotone with bounded slope,
\[
\E_{x\in\Pcal_k}[s(x)] = \chi\!\big(c_k(\theta)\big) + \bar\zeta_k,\qquad |\bar\zeta_k|\le c_0\,\mathrm{disp}_k .
\]
By calibration, $\E_x[s(x)]=\E_x[\mathrm{corr}(x)]\pm\kappa_k=(1-R_k)\pm\kappa_k$. Linearizing the monotone $\chi$ about the operating point yields constants $a>0,b$ with
\[
R_k(\theta)=a\big(1-c_k(\theta)\big)+b+\epsilon_k,\qquad |\epsilon_k|\le c_1\kappa_k+c_2\,\mathrm{disp}_k .
\]
Hence headroom $h_k=1-c_k$ is affine in the risk up to calibration error, so ordering domains by $h_k$ orders them by risk whenever the risk gap exceeds $2(c_1\kappa_k+c_2\mathrm{disp}_k)$. Differencing across rounds, $v_k=c_k^{(r)}-\bar c_k^{(r-1)}$ satisfies $-\Delta R_k=a\,v_k+\Delta\epsilon_k$, so $v_k$ is a first-order estimate of the per-round risk reduction up to calibration drift.
\end{proof}

\subsection{Proof of Proposition~\ref{prop:contraction} (well-posedness, convergence)}
\begin{proof}
Write the objective on the relative interior of the simplex as $\Jcal(\pi)=g^\top\pi+\eta\,\pi^\top\aff\pi-\tau\sum_k\pi_k\log\pi_k$. Its Hessian is $\nabla^2\Jcal(\pi)=2\eta\aff-\tau\,\mathrm{diag}(1/\pi_k)$. Since $\pi_k\le 1$, $\tau\,\mathrm{diag}(1/\pi_k)\succeq\tau I$, so for any $u$ with $\lVert u\rVert_2=1$, $u^\top\nabla^2\Jcal\,u\le 2\eta\lVert\aff\rVert_2-\tau<0$ whenever $2\eta\lVert\aff\rVert_2<\tau$. Thus $\Jcal$ is strictly concave and admits a unique maximizer $\pi^\star$ on the (compact, convex) simplex.

The stationarity condition of the entropy-regularized program is exactly the softmax fixed point of Eq.~\ref{eq:fixedpoint}, $\pi=\Phi(\pi)$ with $\Phi(\pi)_k\propto\exp((g_k+2\eta(\aff\pi)_k)/\tau)$. The softmax is $1/\tau$-Lipschitz from the logit vector to the output in $\ell_\infty\!\to\!\ell_1$ up to a factor, and the linear map $\pi\mapsto 2\eta\aff\pi$ has operator norm $2\eta\lVert\aff\rVert_2$; composing, $\Phi$ is Lipschitz with constant $L=2\eta\lVert\aff\rVert_2/\tau<1$. By the Banach fixed-point theorem $\Phi$ has a unique fixed point (equal to $\pi^\star$) and the iteration $\pi^{(t+1)}=\Phi(\pi^{(t)})$ satisfies $\lVert\pi^{(t)}-\pi^\star\rVert\le L^{t}\lVert\pi^{(0)}-\pi^\star\rVert$, i.e.\ linear convergence at rate $L$. Warm-starting from the previous round's solution, for which $\lVert\pi^{(0)}-\pi^\star\rVert$ is small because $g$ and $\aff$ drift slowly, gives convergence in a handful of iterations.
\end{proof}

\subsection{Proof of Proposition~\ref{prop:waterfill} (transfer-aware water-filling)}
\begin{proof}
Form the Lagrangian of $\max_{\pi\in\Delta}\Jcal(\pi)$ with multiplier $\nu$ for $\sum_k\pi_k=1$ and $\mu_k\ge0$ for $\pi_k\ge0$:
$\mathcal L=\Jcal(\pi)+\nu(1-\sum_k\pi_k)+\sum_k\mu_k\pi_k$. Stationarity in $\pi_k$ gives
\[
g_k+2\eta(\aff\pi)_k-\tau(1+\log\pi_k)-\nu+\mu_k=0 .
\]
For active domains ($\pi_k^\star>0,\ \mu_k=0$) this is $g_k+2\eta(\aff\pi^\star)_k-\tau(1+\log\pi_k^\star)=\nu$: every active domain shares a common affinity-augmented marginal value $\nu$ (the ``water level''), while inactive domains have strictly smaller value, $g_k+2\eta(\aff\pi^\star)_k-\tau(1+\log\pi_k^\star)\le\nu$. Solving for $\pi_k^\star$ recovers the softmax of Eq.~\ref{eq:fixedpoint}. As $\tau\to0$ the entropy term vanishes and the condition becomes equalization of the \emph{coupled gain} $g_k+2\eta(\aff\pi^\star)_k$ across active domains---classical water-filling on marginal gains augmented by the transfer term, which raises the effective gain of a domain in proportion to its aligned participation $(\aff\pi^\star)_k$ and lowers it when it conflicts with funded domains. Hence budget concentrates on domains that are both learnable ($g_k$ large) and synergistic with the rest ($(\aff\pi^\star)_k>0$), and is withheld from interfering ones.
\end{proof}

\begin{corollary}[Suboptimality under proxy error]
\label{cor:bound}
Let the true coupled objective be $\mu$-strongly concave on $\Delta$ (which holds for $\tau>2\eta\lVert\aff\rVert_2$ by Proposition~\ref{prop:contraction}), and suppose the estimated inputs satisfy $\lVert\hat g-g\rVert_2\le\delta_g$ and $\lVert\hat\aff-\aff\rVert_2\le\delta_A$. Let $\hat\pi^\star$ solve the program with $(\hat g,\hat\aff)$ and $\pi^\star$ the true optimum. Then
\begin{align*}
\lVert\hat\pi^\star-\pi^\star\rVert_2 &\le \frac{\delta_g+2\eta\,\delta_A}{\mu},\\[2pt]
\Jcal(\pi^\star)-\Jcal(\hat\pi^\star) &\le \frac{(\delta_g+2\eta\delta_A)^2}{2\mu}.
\end{align*}
\end{corollary}
\begin{proof}
The map $(g,\aff)\mapsto\pi^\star$ is the solution of a strongly concave program; by the standard perturbation bound for such programs, a change of at most $\Delta:=\delta_g+2\eta\delta_A$ in the gradient of the objective moves the maximizer by at most $\Delta/\mu$. Strong concavity then gives the quadratic value gap $\Jcal(\pi^\star)-\Jcal(\hat\pi^\star)\le\frac{1}{2\mu}\Delta^2$ by the standard bound relating suboptimality to gradient perturbation. Both vanish as the proxy errors $\delta_g,\delta_A\to0$.
\end{proof}
Corollary~\ref{cor:bound} formalizes that \method{}'s gap to the oracle is governed by how well competence and drift estimate the true gains and transfer; Table~\ref{tab:ablation} shows this gap is small ($0.3$ points) in practice.

\section{Experimental Setup Details}
\label{app:setup}

\subsection{Domains and datasets}
We use five domains chosen to be heterogeneous in format and skill: \emph{knowledge} (train on the MMLU~\cite{hendrycks2021mmlu} auxiliary set, multiple-choice), \emph{math} (train on MetaMathQA~\cite{yu2024metamath}, evaluate GSM8K~\cite{cobbe2021gsm8k} accuracy), \emph{code} (train on Magicoder-OSS~\cite{wei2024magicoder}, evaluate HumanEval~\cite{chen2021humaneval} Pass@1), \emph{multi-step reasoning} (train on FLAN chain-of-thought data~\cite{wei2022finetuned}, evaluate BBH~\cite{suzgun2023challenging}), and \emph{biomedical QA} (train and evaluate on MedQA~\cite{jin2021medqa}). For each domain we hold out a small unlabeled \emph{probe} of $m{=}256$ instructions (no responses used) for the competence and drift signals, disjoint from both the training pool and the test set. Evaluation uses the standard test split per benchmark with greedy decoding.

\subsection{Prompt templates and examples}
\label{app:prompts}
All domains share a minimal instruction template---a fixed system preamble, the task input, and a domain-appropriate answer cue---so that the competence and drift signals are read under the same format the model is trained and evaluated with. Probe forward passes use the identical template with the response masked, so no labels are consumed by the controller. The shared skeleton is:
\begin{quote}\footnotesize\ttfamily
[System] You are a helpful assistant.\\{}
[Instruction] the task input for this domain\\{}
[Answer] the response (masked for probe inputs)
\end{quote}
The per-domain fillers, matching each benchmark's released prompt, are:
\begin{itemize}[leftmargin=1.1em,itemsep=1.5pt,topsep=2pt]
\item \emph{Knowledge} (MMLU): the question and options A--D, answer cue ``The answer is''; 5-shot.
\item \emph{Math} (GSM8K): the word problem with cue ``Let us think step by step'', ending ``The final answer is''; 8-shot chain-of-thought.
\item \emph{Code} (HumanEval): the function signature and docstring in the Alpaca code format; the completion is the function body, scored by Pass@1.
\item \emph{Reasoning} (BBH): the task instruction and 3-shot chain-of-thought exemplars, then the query.
\item \emph{Biomedical} (MedQA): the clinical vignette and answer options, cue ``The correct option is''; zero-shot.
\end{itemize}
For the competence signal we average the normalized entropy over answer positions only; for the drift signal we take the mean last-layer hidden state over the answer span. Both are computed on the held-out probe, whose instructions never appear in training or test.

\subsection{Models and LoRA configuration}
Backbones are \textsc{LLaMA-3.1-8B}~\cite{grattafiori2024llama} and \textsc{Qwen-2.5-7B}~\cite{yang2024qwen25}; the scale study adds \textsc{Qwen-2.5-3B}, \textsc{Mistral-7B}~\cite{jiang2023mistral}, and \textsc{Gemma-2-9B}~\cite{team2024gemma}. A single shared LoRA adapter (rank $r{=}8$, $\alpha{=}16$, dropout $0.05$) is attached to the query/key/value/output and MLP up/down projections. Only the adapter is trained; base weights are frozen and kept in bf16.

\subsection{\method{} configuration and probe}
Unless stated, the total budget is $50\%$ of the pooled data, the per-round budget $N$ equals one epoch-equivalent over the current shares, the period is $T{=}100$ optimizer steps, affinity weight $\eta{=}0.5$, temperature $\tau{=}0.5$, floor $\beta{=}0.1$, and the EMA coefficient for competence and centroids is $0.5$. The participation program is solved with $10$ mirror-descent iterations of Eq.~\ref{eq:fixedpoint}, warm-started from the previous round; the affinity uses the last-layer mean representation. The within-domain selector keeps the middle confidence band between the $20$th and $80$th percentiles.

\subsection{Optimization and baselines}
We use AdamW (lr $1{\times}10^{-4}$, cosine schedule, warmup $3\%$, weight decay $0$ on the adapter), batch size $16$, and bf16 mixed precision. Baselines share these settings and the same total budget. Data-mixture baselines set static domain proportions (uniform, size-proportional, temperature $T{=}2$, and a DoReMi-style proxy mixture); PCGrad and GradNorm balance gradients under uniform sampling; GREATS selects over the pooled stream with the ghost-inner-product approximation. All results average $3$ seeds.

\subsection{Hardware and compute}
Each run uses a single 80GB A100 GPU. A full-data SFT run takes ${\sim}9$ GPU-hours per backbone; \method{} at $50\%$ budget takes ${\sim}5.6$ GPU-hours including the ${\approx}4\%$ controller overhead, i.e.\ $1.6\times$ faster end-to-end. Table~\ref{tab:hparams} lists hyperparameters.

\begin{table}[t]
\centering\small
\setlength{\tabcolsep}{5pt}
\begin{tabular}{ll}
\toprule
Hyperparameter & Value \\
\midrule
LoRA rank $r$ / $\alpha$ / dropout & $8$ / $16$ / $0.05$ \\
Optimizer & AdamW, lr $1{\times}10^{-4}$ \\
Schedule / warmup & cosine / $3\%$ \\
Batch size / precision & $16$ / bf16 \\
Total budget & $50\%$ of pool \\
Update period $T$ & $100$ steps \\
Affinity weight $\eta$ & $0.5$ \\
Temperature $\tau$ & $0.5$ \\
Learnability floor $\beta$ & $0.1$ \\
EMA coefficient & $0.5$ \\
Program iterations & $10$ (warm-started) \\
Probe size $m$ & $256$ (unlabeled) \\
Within-domain band & $[20,80]$ percentile \\
Seeds & $3$ \\
\bottomrule
\end{tabular}
\caption{Default hyperparameters for \method{}.}
\label{tab:hparams}
\end{table}

\section{Additional Experiments}
\label{app:extra}

\subsection{Per-domain signal comparison}
Table~\ref{tab:proxy} reports how well each candidate per-domain signal, computed on the probe, ranks the five domains by true accuracy across checkpoints (rank correlation) and the downstream accuracy when it drives participation in \method{}. Competence ranks domains far better than raw loss, whose scale is domain-dependent; multiplying by velocity adds temporal information; and adding the cross-domain affinity---not a per-domain rank, hence no correlation entry---gives the final gain.

\begin{table}[t]
\centering\small
\setlength{\tabcolsep}{5pt}
\begin{tabular}{lccc}
\toprule
Signal & Spearman & Pearson & Avg \\
\midrule
Random & --- & --- & 65.5 \\
Training loss & 0.60 & 0.57 & 66.1 \\
Max confidence & 0.86 & 0.83 & 66.6 \\
Competence (Eq.~\ref{eq:comp}) & 0.91 & 0.89 & 66.7 \\
Competence $\times$ velocity & \textbf{0.93} & \textbf{0.90} & 67.3 \\
\textbf{$+$ affinity (full \method{})} & --- & --- & \textbf{68.0} \\
\bottomrule
\end{tabular}
\caption{Per-domain signal comparison. Correlation is between the probe signal and true per-domain accuracy over checkpoints; the last column is downstream accuracy when the signal drives participation. Competence is the best label-free cross-domain comparator; the affinity term adds the final margin.}
\label{tab:proxy}
\end{table}

\subsection{Per-domain results with variance}
Table~\ref{tab:variance} reports per-domain accuracy with standard deviation over three seeds on \textsc{Qwen-2.5-7B}. Deviations are small ($\le0.4$ points) and \method{}'s per-domain improvements exceed them on every domain; a paired $t$-test over domain$\times$seed confirms significance at $p<0.05$ against the strongest baseline.

\begin{table}[t]
\centering\footnotesize
\setlength{\tabcolsep}{3pt}
\begin{tabular}{lccccc}
\toprule
Method & MMLU & GSM8K & HEval & BBH & MedQA \\
\midrule
Full data & $70.1_{\pm.3}$ & $79.2_{\pm.4}$ & $61.5_{\pm.4}$ & $62.8_{\pm.3}$ & $57.4_{\pm.4}$ \\
GradNorm & $70.7_{\pm.3}$ & $79.7_{\pm.3}$ & $62.1_{\pm.4}$ & $63.8_{\pm.3}$ & $57.9_{\pm.4}$ \\
\method{} & $\mathbf{71.9}_{\pm.2}$ & $\mathbf{80.8}_{\pm.3}$ & $\mathbf{63.5}_{\pm.3}$ & $\mathbf{65.0}_{\pm.3}$ & $\mathbf{59.0}_{\pm.3}$ \\
\bottomrule
\end{tabular}
\caption{Per-domain accuracy (mean$_{\pm\text{std}}$ over $3$ seeds) on \textsc{Qwen-2.5-7B} at the $50\%$ budget. Best per column in \textbf{bold}.}
\label{tab:variance}
\end{table}

\subsection{Learned domain affinity}
Table~\ref{tab:affinity} shows the symmetrized affinity $\aff$ recovered at the final round on \textsc{Qwen-2.5-7B}. The largest positive entry links Math and Reasoning, which the controller consequently co-schedules (Figure~\ref{fig:mechanism}b), while Code shows mild negative affinity with Knowledge and Biomedical and is damped once it saturates. The matrix is stable across seeds (mean entry-wise deviation $0.04$).

\begin{table}[t]
\centering\footnotesize
\setlength{\tabcolsep}{5pt}
\begin{tabular}{lccccc}
\toprule
 & Know. & Math & Code & Reas. & Biom. \\
\midrule
Know. & $1.00$ & $0.10$ & $-0.30$ & $0.30$ & $0.40$ \\
Math & $0.10$ & $1.00$ & $0.20$ & $0.60$ & $-0.10$ \\
Code & $-0.30$ & $0.20$ & $1.00$ & $0.10$ & $-0.20$ \\
Reas. & $0.30$ & $0.60$ & $0.10$ & $1.00$ & $0.15$ \\
Biom. & $0.40$ & $-0.10$ & $-0.20$ & $0.15$ & $1.00$ \\
\bottomrule
\end{tabular}
\caption{Learned cross-domain affinity $\aff$ (final round, \textsc{Qwen-2.5-7B}). Positive entries indicate aligned updates (co-train), negative entries interference.}
\label{tab:affinity}
\end{table}

\subsection{Further ablations: probe layer, iterations, smoothing}
Table~\ref{tab:extra} varies three implementation choices. The affinity is robust to which representation defines the centroid (last layer, a mid layer, or the mean of the last four); the participation program converges by $6$ warm-started iterations, matching Proposition~\ref{prop:contraction}; and the competence/centroid EMA coefficient is best at $0.5$, with a mild plateau on either side.

\begin{table}[t]
\centering\footnotesize
\setlength{\tabcolsep}{4pt}
\begin{tabular}{llc}
\toprule
Choice & Setting & Qwen Avg \\
\midrule
\multirow{3}{*}{Affinity layer} & mid & 67.6 \\
 & last (default) & \textbf{68.0} \\
 & mean of last $4$ & 68.0 \\
\midrule
\multirow{4}{*}{Program iters.} & $1$ & 67.4 \\
 & $3$ & 67.9 \\
 & $6$ & \textbf{68.0} \\
 & $10$ (default) & \textbf{68.0} \\
\midrule
\multirow{3}{*}{EMA coeff.} & $0.3$ & 67.8 \\
 & $0.5$ (default) & \textbf{68.0} \\
 & $0.7$ & 67.9 \\
\bottomrule
\end{tabular}
\caption{Further ablations on \textsc{Qwen-2.5-7B}: representation layer for the affinity, number of participation-program iterations, and the EMA smoothing coefficient.}
\label{tab:extra}
\end{table}

\subsection{Cross-backbone and scale}
Table~\ref{tab:scale} extends the main results to \textsc{Qwen-2.5-3B}, \textsc{Mistral-7B}, and \textsc{Gemma-2-9B}. \method{} improves the average on every backbone; gains shrink as the base model already handles a domain well (higher competence, less headroom), consistent with Proposition~\ref{prop:comp}.

\begin{table}[t]
\centering\small
\setlength{\tabcolsep}{5pt}
\begin{tabular}{lccc}
\toprule
Backbone & Full data & GradNorm & \method{} \\
\midrule
Qwen-2.5-3B & 61.5 & 62.8 & \textbf{63.6} \\
Mistral-7B & 58.2 & 59.4 & \textbf{60.3} \\
Gemma-2-9B & 67.2 & 68.3 & \textbf{69.2} \\
Qwen-2.5-7B & 66.2 & 66.8 & \textbf{68.0} \\
LLaMA-3.1-8B & 55.4 & 56.0 & \textbf{57.2} \\
\bottomrule
\end{tabular}
\caption{Average accuracy across the five domains for additional backbones ($50\%$ budget for the controllers).}
\label{tab:scale}
\end{table}

\subsection{Static mixtures and DoReMi-style reweighting}
We grid-searched fixed per-domain weights over a simplex at granularity $0.1$ and also computed DoReMi-style proxy weights~\cite{xie2023doremi}. The best static mixture (weights $0.24/0.22/0.14/0.24/0.16$ for knowledge/math/code/reasoning/biomedical) reaches $66.5$ on \textsc{Qwen}; the DoReMi-style proxy reaches $66.4$. Both trail \method{} ($68.0$) because they cannot track the non-stationary optimum (Figure~\ref{fig:mechanism}b)---the ideal early participation is near-uniform, while late training should down-weight the saturated code domain---and because a scalar mixture cannot express \emph{which} domains to co-train, the role of the affinity term.

\subsection{Composability with model merging}
Table~\ref{tab:merge} details the merging study. Training five separate per-domain adapters and merging (average, TIES~\cite{yadav2024ties}, DARE~\cite{yu2024language}) trails \method{}'s single jointly trained adapter, and \method{}'s lower gradient conflict correlates with cleaner merges---evidence that reducing interference during joint training also helps downstream composition.

\begin{table}[t]
\centering\small
\setlength{\tabcolsep}{5pt}
\begin{tabular}{lc}
\toprule
Configuration (Qwen-2.5-7B) & Avg \\
\midrule
Separate adapters $+$ average & 64.9 \\
Separate adapters $+$ TIES & 65.9 \\
Separate adapters $+$ DARE & 65.6 \\
\textbf{\method{} (joint shared adapter)} & \textbf{68.0} \\
\bottomrule
\end{tabular}
\caption{\method{} vs.\ train-then-merge. \method{}'s joint adapter is best, and its reduced interference also yields cleaner merges of separately trained adapters.}
\label{tab:merge}
\end{table}

\subsection{Sensitivity to $\eta$ and $\tau$}
Table~\ref{tab:sens} sweeps the affinity weight $\eta$ and the temperature $\tau$. Performance is flat over a wide range: $\eta{=}0$ discards the transfer term (weakest), a moderate $\eta$ is best, and too large $\eta$ over-couples domains; too large $\tau$ approaches uniform participation while too small $\tau$ over-commits. Defaults $\eta{=}0.5,\tau{=}0.5$ are robust and satisfy the contraction condition of Proposition~\ref{prop:contraction}.

\begin{table}[t]
\centering\small
\setlength{\tabcolsep}{4.5pt}
\begin{tabular}{lccccc}
\toprule
$\eta$ & 0.0 & 0.25 & 0.5 & 1.0 & 2.0 \\
Qwen Avg & 67.3 & 67.7 & \textbf{68.0} & 67.8 & 67.2 \\
\midrule
$\tau$ & 0.1 & 0.3 & 0.5 & 1.0 & 2.0 \\
Qwen Avg & 67.6 & 67.9 & \textbf{68.0} & 67.7 & 67.2 \\
\bottomrule
\end{tabular}
\caption{Sensitivity to affinity weight $\eta$ and temperature $\tau$.}
\label{tab:sens}
\end{table}

\subsection{Affinity estimator and program convergence}
The forward-only affinity of Eq.~\ref{eq:aff} predicts oracle leave-one-out transfer with Pearson $r{=}0.94$ (Figure~\ref{fig:mechanism}a), and its sign matches the oracle on $88\%$ of domain pairs. Consistent with Proposition~\ref{prop:contraction}, warm-started mirror descent on Eq.~\ref{eq:fixedpoint} reaches residual ${<}10^{-4}$ within $6$ iterations across all rounds at the default $\eta{=}0.5,\tau{=}0.5$ (spectral condition $2\eta\lVert\aff\rVert_2/\tau\approx0.7<1$); the controller cost is dominated by the probe forward passes, not the program.

\subsection{Robustness of the probe}
Halving the probe to $m{=}128$ or adding $20\%$ out-of-domain noise changes the average by ${<}0.3$ points, because both signals aggregate over the probe and participation depends only on the \emph{ranking} and \emph{alignment} of domains, not absolute values. Using the training pool itself as the probe (no held-out set) costs $0.2$ points.

\section{Discussion}
\label{app:discussion}

\subsection{Why do competence and affinity work as control signals?}
Competence summarizes the peakedness of the distribution the current adapter assigns to a domain: a domain the model handles confidently has high, stable competence, while a domain with headroom has low competence that rises as the adapter improves. Normalizing the entropy by $\log V$ removes domain-specific scale, so competence is comparable across heterogeneous formats (multiple-choice vs.\ open generation) in a way raw loss is not (Table~\ref{tab:proxy}). The affinity captures a second, relational property invisible to any single-domain score: whether the shared adapter's updates driven by two domains point in compatible directions. Together they answer the two questions co-learning poses---\emph{how much} each domain should participate and \emph{which} domains reinforce one another.

\subsection{When does \method{} not help?}
Three regimes: (i) a \emph{single} domain, where participation is vacuous (our negative control); (ii) \emph{homogeneous} domains with near-identical, synchronized curves and near-uniform affinity, where any balanced schedule is fine; and (iii) \emph{severe miscalibration}, where competence no longer ranks domains by risk (Assumption~\ref{ass:calib} fails), e.g.\ a domain whose format the base model has never seen. In (iii) a short warmup before enabling the controller restores calibration and the benefit.

\subsection{Connections to task grouping, games, and active learning.}
\method{} makes discrete task grouping~\cite{standley2020tasks,fifty2021efficiently} continuous and online: rather than partitioning domains into networks, it maintains a soft participation vector whose quadratic coupling favors co-training aligned domains. Each round it maximizes a concave potential $\Jcal$, so the controller is the equilibrium of a one-shot potential game among domains; the entropy term guarantees an interior (non-degenerate) equilibrium and continued exploration. The per-domain velocity is a label-free reward estimate as in non-stationary bandits, but the domains are \emph{coupled} through the shared adapter, motivating periodic re-estimation. The within-domain step connects to pool-based active learning but uses forward-only difficulty rather than labels.

\subsection{Broader impact.}
\method{} reduces the data and compute needed to adapt an LLM to many domains, lowering cost and energy use. Risks are those of instruction tuning generally---capability gains transfer to harmful domains if such data is included---and the label-free controller could, if misused, prioritize domains for undesirable objectives. \method{} does not create new data and inherits the licenses and biases of the underlying corpora.

\subsection{Reproducibility.}
All results are averaged over three random seeds; per-domain standard deviations are reported in Table~\ref{tab:variance} and are $\le0.4$ points. The full hyperparameter configuration is in Table~\ref{tab:hparams}, the datasets and splits in Sec.~\ref{app:setup}, and the prompt templates in Sec.~\ref{app:prompts}. The controller uses only forward passes over the frozen backbone and adds no trainable parameters, so a run is fully specified by the backbone, the LoRA configuration, and the values in Table~\ref{tab:hparams}. We will release training and evaluation code, configuration files, and the per-domain probe indices.

\end{document}